\documentclass{ssl}
\usepackage{graphicx} 
\usepackage[toc,page]{appendix}
\usepackage[english]{babel}
\usepackage{amsmath, amsfonts, amsthm, amssymb}
\usepackage{array,booktabs, graphicx, tabularx, ,pdflscape, siunitx, makecell}

\title{Form and Function: Machine Unlearning as a Problem of Misaligned States}
\subtitle{Machine unlearning is defined by its performance, but this leaves another component of the model unexamined: its geometry.}
\author{Kennon Stewart\affilnum{1,2}}
\affiliation{\affilnum{1} Second Street Labs, Detroit, MI, USA\\
\affilnum{2} Department of Statistics, University of Michigan, Ann Arbor, MI, USA}
\corrauth{Kennon Stewart, Second Street Labs}
\email{kennon@secondstreetlabs.io}
\date{April 2026}
\version{v1}
\preprintnumber{2026-002}
\lablogo{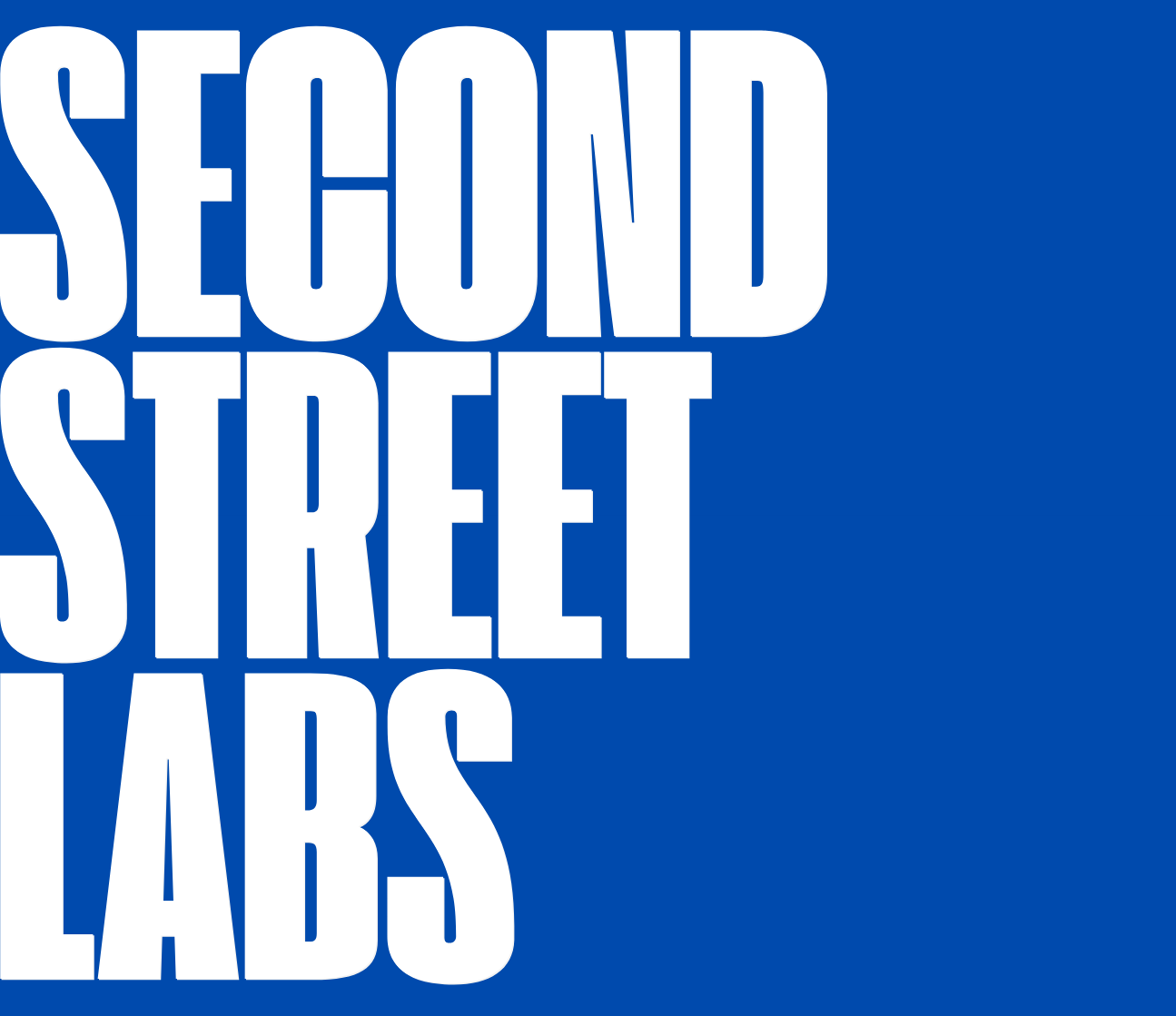}

\begin{document}
\begin{abstract}
We formulate machine unlearning for online L-BFGS as a counterfactual state-alignment problem. Given an actual event stream and a deletion-edited counterfactual stream, the target of unlearning is the optimizer state that would have arisen had the deleted samples never been processed. We introduce state-aware metrics that separately measure parameter error, memory-operator error, combined state error, and update-direction error. The memory metric compares the inverse-Hessian actions induced by the o-L-BFGS memory, rather than treating curvature pairs as of finite influence. Under convexity assumptions, we derive a recursive bound on counterfactual state deviation. We then evaluate a state-aware benchmark of deletion interventions, including memory-only and parameter-only corrections, against an counterfactual oracle model. These results show that unlearning for online L-BFGS is not merely a parameter-correction problem: it requires alignment with a realizable counterfactual optimizer state.
\end{abstract}
\maketitle

\section{Introduction}
Machine unlearning is an increasingly popular area of research, but has an evaluation problem. It's simple enough to conceive of the goal for unlearning a piece of data: the unlearned model should act identically to the model having never seen the data at all. But the function of a model is very different from its form; the description of one without the other is simply incomplete.

This is especially the case for stateful optimizers, which are composed of more than parameters and predictions. Optimizer states exhibit momentum, preconditioned states, and curvature memory that are undescribed by the model's predictions.

We study this issue in the online L-BFGS optimizer specifically for the task of machine unlearning. The optimizer maintains a finite memory of curvature pairs used to construct an inverse-Hessian approximation. This lasting geometric memory tangibly changes the update direction applied to future gradients. A deletion request therefore raises a structural question. It is not enough to ask whether the current parameter vector performs like the counterfactual parameter vector. We must also ask whether the coupled optimizer state, consisting of parameters and curvature memory, could have arisen from the deletion-edited event stream.

We formulate this as a counterfactual state-alignment problem. Given an actual event history and a deletion-edited counterfactual history, the unlearning target is the optimizer state generated by the counterfactual stream. This reframes unlearning as an \textit{ontological} question about the state of the learning system. Specifically, can it assume \textit{both the form and function} of our counterfactual ideal?

\section{Contributions}
\begin{enumerate}
    \item We reframe unlearning as a state-alignment task where the optimizer aims to recover the counterfactual state having never learned deleted data.
    \item We demonstrate the existence of two forms of memory in a second-order optimizer: the curvature information explicitly encoded into the state, and the indirect memory that persists once a curvature point is "forgotten".
    \item We introduce unlearning benchmarks that measure state deviation, specifying unlearning for the entire optimizer state.
    \item We provide theoretical and empirical evidence that parameter-only and memory-only deletion interventions are insufficient for reliable unlearning, while replay-based reconstruction can recover a realizable counterfactual parameter-memory state over a finite horizon.
\end{enumerate}

\section{Related Work}
\subsection{Unlearning}
Machine unlearning literature is broadly divided into two methods: exact and approximate \cite{xu_machine_2024}.

Exact unlearning removes unlearned data from the model completely. Work by Bourtoule et al. \cite{bourtoule_machine_2020} proposed the SISA method, which partitions the training set into isolated components when training the model. In order to erase data, only a segment of the production model need be pruned, allowing for total erasure of the data. But such methods have been shown to reduce the utility of the model and requires the training set to be kept in-memory. This makes it infeasible for memory-constrained and online environments.

Approximate unlearning brings the unlearned model within a statistical distance of the counterfactual. Suriyakumar and Wilson provide a particularly comprehensive overview \cite{suriyakumar_algorithms_2022}. The literature commonly recommends parameter perturbations to achieve some guaranteed similarity to the counterfactual, with complete retraining being the exact and expensive comparator \cite{sekhari_remember_2021}. Later works allows for compressed representations of curvature memory that reduce time and space complexities \cite{qiao_hessian-free_2025} and even online unlearning with regret guarantees \cite{hu_online_2025, shen_machine_2025}.

Prior work exposed the fallibility of current approximate unlearning methods, notably the presence of residual information in the unlearned model \cite{liu_threats_2025, chen_when_2021, stewart_shape_2026}. This demonstrates the insufficiency of unlearning evaluation by comparing performance, and even persists under the more restrictive parameter identify requirement. 

There remains an opportunity for an unlearning benchmark specifically for high-dimensional learning. Though black-box metrics (like the predictive performance of the optimizer) is easily evaluated, it doesn't fully capture the internal dynamics of the unlearned model state, which may differ substantially from that of the counterfactual \cite{jeon_information_2024, stewart_shape_2026}.

\subsection{Second-Order Optimization}
Gradient descent is a family of optimization techniques popular for their interpretability. Numerous works have noted weakness of first-order descent in poorly-conditioned environments, where saddle points and shallow surfaces can slow descent. Second-order optimization improves on prior work by collecting curvature information and accounts for ill-conditioning in the loss surface \cite{cesa-bianchi_prediction_2006}.

Newton's Method is the canonical second-order optimizer, relying explicitly on the Hessian of the loss function to account for conditioning. The combination of gradient and curvature information allows for identical performance on any affine transformation of a loss surface \cite{boyd_convex_2023}. Though the method is useful in particular learning tasks, it's infeasible in high-rank learning where $m >> n$ due to its expensive $O(d^2)$ Hessian inversion operation.

Stochastic and quasi-newton methods emerged as alternatives that addressed such intense calculation. They replace a full Hessian inversion with approximations to varying degrees of compression. The popular BFGS method maintains a Hessian-like preconditioner that is used in place of the Hessian itself \cite{byrd_limited_1995}. This was succeeded by the L-BFGS method, designed specifically for memory-constrained environments, which recursively updates an approximation of the Hessian in its inverse form, negating the need for repeated inversions. More recent work has focused on maintaining these inverse Hessian approximations stochastically, allowing for a method that approaches online learning in its function with a constant $O(\tau d)$ memory constraint \cite{qiao_hessian-free_2025}.

\section{Problem Setup}
\subsection{Sequential Learning and Unlearning}
We define a stream of events, which is produced by some dynamic data-generating process.

\begin{definition}[Actual event stream]
Let $\mathcal X$ be the sample space and let $\mathcal I$ be a countable set of sample indices. An event at time $t$ is a tuple
$$
    e_t = (o_t,i_t,z_t),
$$
where $o_t \in \{\mathrm{insert},\mathrm{delete}\}$ is the event type, $i_t \in \mathcal I$ is a sample index, and $z_t \in \mathcal X \cup \{\varnothing\}$ is the sample payload. If $o_t=\mathrm{insert}$, then $z_t\in\mathcal X$ is inserted into the learner. If $o_t=\mathrm{delete}$, then $z_t=\varnothing$ and the event requests removal of the previously inserted sample with index $i_t$.

The actual event history up to time $t$ is $H_t = (e_1,\ldots,e_t).$
\end{definition}

We distinguish the actual event history seen by the learner with the counterfactual, which excludes the points to have been deleted. The ideal unlearned model would resemble the counterfactual not only in model behavior, but model geometry as well.

\begin{definition}[Deletion-edited counterfactual history]
For a deletion set $D\subseteq I_t^{+}$, the deletion-edited counterfactual
history is the subsequence
$$
    H_t^{-D}
    =
    \bigl(e_s : s\le t,\ i_s\notin D\bigr),
$$
with the original temporal order preserved.

For the realized deletion requests by time $t$, we write
$$
    H_t^{-} := H_t^{-D_t}.
$$
\end{definition}

The two event streams produce separate learner states, whose differences indicate the residue of unlearned information.

\begin{definition}[Actual and counterfactual learner states]
Let $\Phi_t$ denote the optimizer update map. The actual learner state evolves as
$$
    \theta_{t+1} = \Phi_t(\theta_t,e_t),
    \qquad \theta_t=(w_t,Z_t).
$$

For a deletion set $D$, the counterfactual learner state $\theta_t^{-D}$ is the
state obtained by applying the same optimizer recursively to the edited history
$H_t^{-D}$:
$$
    \theta_t^{-D} = \Phi_{H_t^{-D}}(\theta_0).
$$
\end{definition}

In the case of the online L-BFGS optimizer, the optimizer state $Z_t$ contains a finite set of curvature pairs $\{ (v_i, r_i) \}_{i=t-\tau}^{t-1}$ used to approximate the curvature of the loss surface. This creates an inherent forgetting mechanism for unlearning data older than the curvature window $\tau$. So long as the learner is sufficiently insensitive with respect to the deletion candidates, the explicit representation of the unlearned data is erased when the curvature pairs are updated.

\begin{definition}[Unlearning operator]
An unlearning operator at time $t$ is a map $\mathcal U_t:\mathcal Z \times 2^{\mathcal H_t} \to \mathcal Z$
that takes the current optimizer state $Z_t$ and a deletion set $U_t$ and returns an unlearned state
$$
\widetilde Z_t = \mathcal U_t(Z_t,D_t).
$$
\end{definition}

\begin{definition}[State-alignment error]
\label{def:state-alignment-error}
Given a metric $d_{\mathcal Z}$ on learner states, the state-alignment error is
$$
\Delta_t = d_{\mathcal Z}(\widetilde \theta_t,\theta_t^{-U}),
$$
where $\theta_t^{-U}$ is the counterfactual optimizer state generated by the deletion-edited history.
\end{definition}

The state deviation bound is a new way of measuring the alignment of an unlearned model, but differs from the probabilistic bounds used to define certified machine unlearning. In the case of states that deviate by at most $\alpha_t,$ we can write the state deviation bounds with respect to canonical certified unlearning definitions.

\begin{definition}[$(\epsilon,\delta)$-State Certified Unlearning]
\label{def: state-certified-unlearning}
Assume that for all deletion sets $U$ with $|U| \leq m,$
$$
    \Pr[d_{\Theta}(\bar \theta_{t}^{U}, \theta_{t}^{-U}) \leq \alpha] \geq 1 - \beta.
$$
The randomized unlearning operation satisfies $(\varepsilon, \delta + \beta)$-state-certified unlearning if 
$$
\sigma \geq \frac{
    \alpha \sqrt{2 \log{(1.25 / \delta)}}
}{
    \epsilon
}.
$$
\end{definition}

In application, the state deviation bounds can be substituted for optimizer-specific unlearning certificates. This describes the strength and shape of noise required for injection to certify indistinguishability from the counterfactual ideal. The connection relies on the contractivity assumption for the optimizer descent. Specifically, the calibrated noise injection is determined by the highest possible state deviation between the observed and counterfactual models, $\alpha_t.$

We discuss extensions in the Appendix, including the K-Norm Gradient Mechanism, which provides a noise certificate in the same metric space as the optimizer.

\section{Theoretical Results}
\subsection{Unlearning with an o-LBFGS optimizer}
On the other hand, deleted information may persist in the trajectory of later curvature pairs. This stems from the fact that the state of the optimizer at time $t$ influences the path and state of later curvature updates $\{ (v_i, r_i) \}_{i=t-\tau}^{t-1}$. We demonstrate that incoming information, reflecting the counterfactual, will bound and diminish the influence of the perturbation with sufficient time.

\begin{proposition}[Direct memory is bounded for finite curvature pairs]
For an o-LBFGS optimizer with memory window $\tau$, any curvature pair generated at time $s$ is no longer stored in the direct memory of $Z_t$ for all $t>s+\tau$.
\end{proposition}

A notable difference between our method and that of prior work is the metric used to diagnose the model's misalignment with the counterfactual. Prior work emphasizes regret and empirical risk as indicators for model recovery \cite{qiao_hessian-free_2025, sekhari_remember_2021}. We note that these methods don't encompass the entire model state, where residual information may persist \cite{chen_when_2021, bertran_reconstruction_2024, cooper_machine_2025, stewart_shape_2026}.

For the sake of brevity, we relegate the standard convexity and smoothness assumptions to Appendix \ref{sec:convexity-assumptions}. We largely follow those outlined in standard stochastic optimization and unlearning literature \cite{qiao_hessian-free_2025, byrd_limited_1995}. We restrict our attention to convex optimization, but explore nonconvex optimizers in future work.

\begin{lemma}[One-step state-error recursion]
\label{lem:state-error-recursion}
Suppose the optimizer update map is $\rho$-contractive in $d_{\mathcal Z}$ and the deletion-induced perturbation at time $t$ is bounded by $\eta_t\epsilon_t$. Then
$$
\Delta_{t+1}
\le
\rho \Delta_t + \eta_t\epsilon_t.
$$
\end{lemma}
\begin{proof}
The state alignment error decomposes into two pieces: the difference in updates in the unlearned model, and that in the counterfactual model.

\begin{align*}
  \Delta_t &= d_\theta(\tilde \theta_{t+1}, \theta_{t+1}) \\
    &= d_{\theta}(\phi_t(\tilde \theta_t), \phi_t^{-U}(\theta_t^{-U})) \\
    &\leq d_{\theta}(\phi_t(\tilde \theta_t), \phi_t(\theta_{t}^{-U})) + d_{\theta}(\phi_t(\theta_t^{-U}), \phi_t^{-U}(\theta_{t}^{-U})) \\
    &\leq \rho \Delta_t + \eta_t \epsilon_t
\end{align*}
\end{proof}

Which justifies the state deviation bounds for the recursive optimizer state, relying on a discrete Gronwall inequality \cite{liao_discrete_2018}.

\begin{theorem}[Bounded counterfactual state deviation for finite-memory o-LBFGS]
\label{thm: bounded-counterfactual-state-deviation}
Under strong convexity, smoothness, bounded gradients, bounded curvature pairs, and contractive o-L-BFGS updates, the state-alignment error satisfies
$$
\Delta_t
\le
\rho^{t-t_0}\Delta_{t_0}
+
\sum_{s=t_0}^{t-1}\rho^{t-1-s}\eta_s\epsilon_s.
$$
The optimizer stores only $O(\tau d)$ memory.
\end{theorem}

The theorem above decomposes the counterfactual error into two parts: the initial disruption caused by deletion and the residual memory in the optimizer state. As mentioned, the error persists in the counterfactual model for a period of time bounded by the number of curvature pairs. The sum $\sum_{s=t_0}^{t-1}\rho^{t-1-s}\eta_s\epsilon_s$ represents the information stored in the indirect memory, the effects of which decay geometrically with the strength of the map's contraction.

\subsection{Defining Certificates on the Optimizer State}
Since we have bounds on the state deviation of the optimizer state, we can then define the conditions under which unlearning is certified not only for the parameter, but for the entire optimizer state, $\theta_t.$ 

To do this, we step back from unlearning as a deterministic operation. We randomize the event stream over which our model learns (and unlearns) before injecting a noise to the parameter state that is calibrated to the one-step recursion error defined in Theorem \ref{thm: olbfgs-state-certificate}.

\begin{theorem}[o-LBFGS Unlearning Certificate]
\label{thm: olbfgs-state-certificate}
The randomized state $\tilde \theta_{t}^{U} = \bar \theta_{t}^{U} + \epsilon$ where $\epsilon \sim \mathcal N(0, \sigma^2I)$ is certifiably unlearned if 
$$
\sigma_t \geq \frac{
    \alpha_t(m,\tau) \sqrt{2 \log{(1.25 / \delta)}}
}{
\epsilon
}
$$
where $\alpha_t(m,\tau) =\rho^{t - t_0}\Delta_{t_0} + \sum_{s = t_0}^{t-1} \rho^{t-1-s} \eta_{s} \epsilon_{s}(m, \tau).$ Exact unlearning is only obtained when $\alpha_t(m,\tau) = 0.$
\end{theorem}

Theorem \ref{thm: olbfgs-state-certificate} describes the noise injection needed to ensure certified unlearning. This surpasses prior parameter-only noise injections by tuning the noise to the optimizer's set of curvature pairs. We suspect that similar certificates are obtainable for other optimization methods with bounded state deviations.

The primary outcome is future state-error area under the curve,
$$
\mathrm{AUC}
=
\sum_{k=0}^{H}
E_\Theta(t_{\mathrm{del}}+k),
$$
which measures cumulative deviation from the counterfactual trajectory. The experiment code also reports initial state error, final state error, future parameter-trajectory error, update-direction-error AUC, direct memory mass, average future loss, and intervention cost measured by replayed events, additional gradient evaluations, and wall-clock time.

\begin{figure}
    \centering
    \includegraphics[width=1\linewidth]{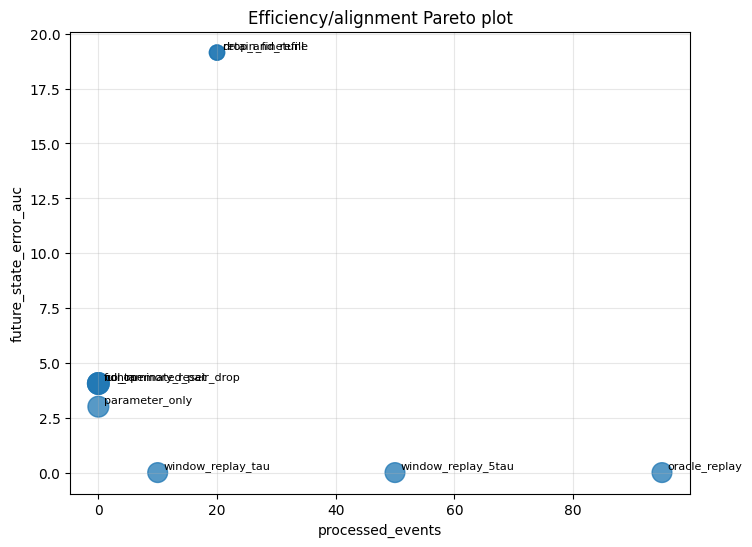}
    \caption{The efficiency of the deletion operation depends on both the size of the deleted dataset as well as its relevance to direct memory. The oracle replay baseline is completely effective, but requires the entire counterfactual event stream. For a fraction of the data, exact counterfactual recovery is possible when the deleted data is still in direct memory. As the deleted data passes from direct to indirect memory, the future state error increases and reduces efficiency.}
    \label{fig:experiment-2-quadratic-efficiency-plot}
\end{figure}

\begin{landscape}
\begin{table}[t]
\centering
\caption{
Deletion-mode mechanism evidence. 
Recent deletions test explicit contamination of the o-LBFGS curvature memory, while random and high-gradient deletions often have zero direct memory mass and therefore test indirect trajectory dependence. 
The replay methods exactly recover recent deletions because the relevant counterfactual history lies inside the replay window. 
For random and high-gradient deletions, direct memory mass is typically zero, yet non-replay methods still exhibit large future state-error AUC, indicating that deleted data can persist through the induced parameter trajectory and subsequently generated curvature pairs.
}
\label{tab:deletion-mode-mechanism}
\small
\begin{tabularx}{\textwidth}{llrrrr}
\toprule
Deletion mode 
& Method 
& \makecell{Direct\\memory mass}
& \makecell{Initial\\state error}
& \makecell{Future\\state AUC}
& \makecell{AUC ratio\\vs. no-op} \\
\midrule
Recent 
& No-op deletion 
& 5.000 & 330.537 & 444010.507 & 1.000 \\
Recent 
& Parameter-only 
& 5.000 & 122.800 & 436623.291 & 0.876 \\
Recent 
& Contaminated pair drop 
& 0.000 & 316.275 & 448272.217 & 0.940 \\
Recent 
& Window replay $\tau$ 
& 0.000 & 0.000 & 0.000 & 0.000 \\
Recent 
& Window replay $5\tau$ 
& 0.000 & 0.000 & 0.000 & 0.000 \\
\midrule
Random 
& No-op deletion 
& 0.000 & 187.849 & 449573.044 & 1.000 \\
Random 
& Parameter-only 
& 0.000 & 322.062 & 453589.582 & 1.016 \\
Random 
& Contaminated pair drop 
& 0.000 & 187.853 & 449573.039 & 1.000 \\
Random 
& Window replay $\tau$ 
& 0.000 & 205.209 & 434590.243 & 1.000 \\
Random 
& Window replay $5\tau$ 
& 0.000 & 0.007 & 0.891 & 0.694 \\
\midrule
High-gradient 
& No-op deletion 
& 0.000 & 203.378 & 452026.482 & 1.000 \\
High-gradient 
& Parameter-only 
& 0.000 & 387.741 & 438854.920 & 1.000 \\
High-gradient 
& Contaminated pair drop 
& 0.000 & 203.378 & 458399.960 & 1.000 \\
High-gradient 
& Window replay $\tau$ 
& 0.000 & 256.555 & 443910.442 & 1.000 \\
High-gradient 
& Window replay $5\tau$ 
& 0.000 & 0.009 & 1.055 & 0.996 \\
\bottomrule
\end{tabularx}
\end{table}
\end{landscape}

\begin{landscape}
\begin{table}[t]
\centering
\caption{
We simulate a stream of learning and unlearning requests following a quadratic loss model.
The quadratic ensures a predictably convex environment, while the ridge-logistic setting tests the same interventions under harder nonlinear dynamics. 
We see in both cases that local state edits do not reliably reconstruct the counterfactual optimizer. Replay-based methods, though approximate, empirically reduce future state error. 
}
\label{tab:stream-type-stress-test}
\small
\begin{tabularx}{\textwidth}{llrrrr}
\toprule
Stream 
& Method 
& \makecell{Future\\state AUC}
& \makecell{AUC ratio\\vs. no-op}
& \makecell{Exact recovery\\rate}
& \makecell{Final\\state error} \\
\midrule
Quadratic 
& No-op deletion 
& 1.137 & 1.000 & 0.000 & 0.000004 \\
Quadratic 
& Parameter-only 
& 1.062 & 0.871 & 0.000 & 0.000003 \\
Quadratic 
& Contaminated pair drop 
& 1.221 & 1.000 & 0.000 & 0.000004 \\
Quadratic 
& Window replay $\tau$ 
& 0.022 & 1.000 & 0.333 & 0.000000 \\
Quadratic 
& Window replay $5\tau$ 
& 0.000 & 0.000 & 0.556 & 0.000000 \\
\midrule
Ridge logistic 
& No-op deletion 
& 1012737.545 & 1.000 & 0.000 & 1113.715 \\
Ridge logistic 
& Parameter-only 
& 1046068.869 & 0.999 & 0.000 & 1085.240 \\
Ridge logistic 
& Contaminated pair drop 
& 996251.965 & 1.000 & 0.000 & 1070.716 \\
Ridge logistic 
& Window replay $\tau$ 
& 925329.661 & 0.996 & 0.333 & 840.374 \\
Ridge logistic 
& Window replay $5\tau$ 
& 0.000 & 0.000 & 0.556 & 0.000 \\
\bottomrule
\end{tabularx}
\end{table}
\end{landscape}

\section{Experimental Results}

We evaluate whether data deletion in online L-BFGS can be treated as a parameter-correction problem, or whether successful unlearning requires recovery of the full optimizer state. The experiments show that deletion influence is not localized to the parameter vector. Instead, deleted data affects the coupled state
$$
    \theta_t = (w_t, Z_t),
$$
where $w_t$ is the parameter vector and $Z_t$ is the finite memory of L-BFGS curvature pairs. This coupling matters because future updates depend on both terms. This is important because prior updates impact the path of the optimizer in state space. An intervention that corrects only $w_t$, or only $Z_t$, generally produces a state that may be closer to the counterfactual along one coordinate, but is not itself a state that would have come from the counterfactual stream of events.

We identify three clear patterns across the experiments. First, parameter alignment understates the persistence of deleted information. Second, curvature memory creates a deletion horizon: even after directly contaminated curvature pairs leave the finite memory window, their influence may persist indirectly through later parameters, later curvature pairs, and future update directions. Third, exact recovery is possible only when the intervention reconstructs the joint parameter-memory state, either by oracle replay or by replay over a window that contains the deletion-relevant optimizer history.

\begin{figure}
    \centering
    \includegraphics[width=1\linewidth]{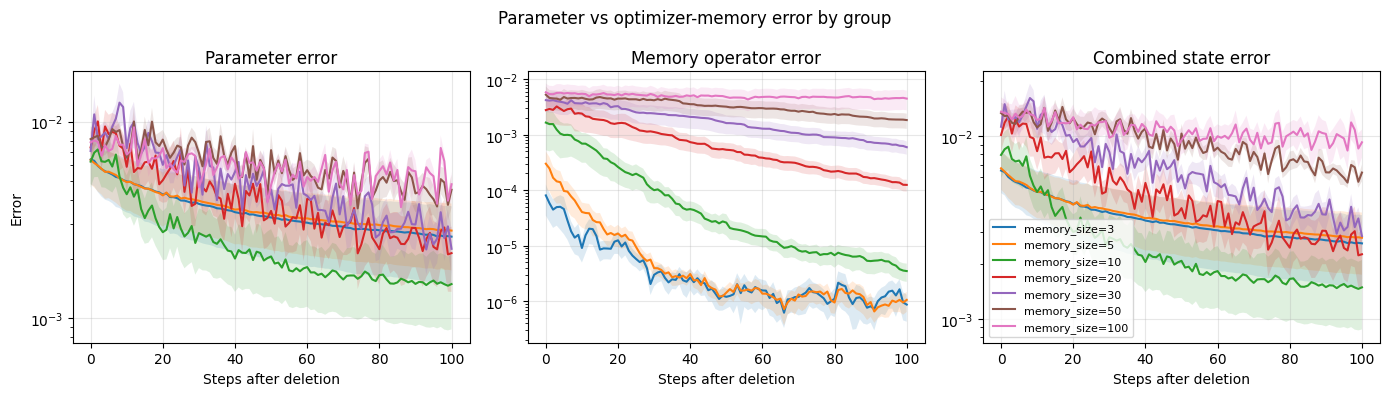}
    \caption{The size of the memory impacts both learning and unlearning efficiency. 10 curvature points yields the strongest parameter alignment, yet shows a consistent increase in the memory operator error. Similarly, a deeper curvature memory visually corresponds to higher memory operator errors, though all interventions show a descent to the counterfactual at rates inversely proportional to the memory size.}
    \label{fig:experiment-1-quadratic-state-error-decomposition}
\end{figure}

\subsection{Parameter alignment understates the persistence of deleted information.}
We begin by comparing parameter error and memory-operator error after deletion. If unlearning were purely a parameter-level problem, then interventions that bring $w_t$ close to the counterfactual parameter $w_t^{-U}$ would also induce future behavior close to the counterfactual trajectory. The results reject this interpretation, as shown in Figure \ref{fig:experiment-1-quadratic-state-error-decomposition}. Methods that intervene on only one component of the optimizer state can reduce one discrepancy while leaving the other unresolved.

This is most visible in the comparison between parameter-only correction, memory-only correction, and replay-based methods. The oracle replay baseline produces zero initial state error, zero final state error, and zero future trajectory error because it reconstructs the deletion-edited counterfactual directly. By contrast, the parameter-only intervention applies a correction to $w_t$ while leaving $Z_t$ unchanged. This creates a parameter-memory pair that is internally inconsistent: the parameter has been moved as if the deleted data were absent, but the curvature memory still encodes update information from the original trajectory. This produces the largest future trajectory error among the non-oracle methods, with future state-error AUC approximately $5.62 \times 10^7$, compared with $1.69 \times 10^7$ for no-op deletion.

This result clarifies why parameter-space deletion is not enough for online L-BFGS. The update direction is computed from both the current gradient and the inverse-Hessian approximation induced by the memory state. If the parameter vector is corrected but the curvature state is not, the next update is not the update that the counterfactual optimizer would have taken. The intervention may therefore remove a first-order trace of the deleted data while preserving a second-order trace in the optimizer memory.

\begin{figure}
    \centering
    \includegraphics[width=1\linewidth]{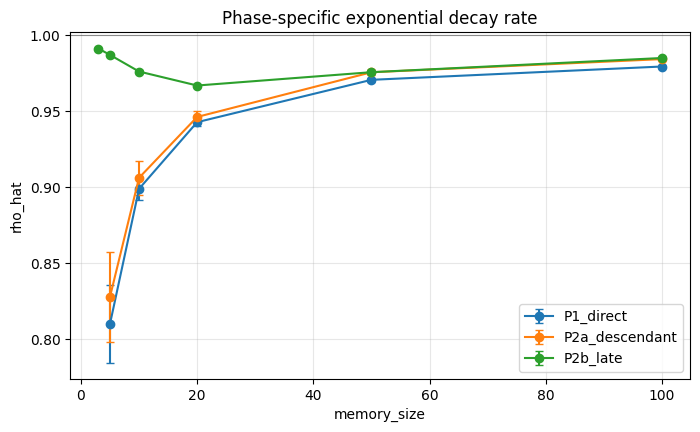}
    \caption{We plot the rate at which the state deviation decays for three separate phases of post-deletion dynamics, specifically for the case of the \textbf{logistic loss surface}. We see that for a fixed number of deleted points, optimizers with fewer curvature points approach the counterfactual faster than those with more curvature information. The optimizer with 20 curvature pairs begins with a decay rate greater than 1, indicating that the optimizer temporarily increases its state deviation for a period directly following deletion. But when the depth of memory increases to 50, then the optimizer decays slowly consistently through the entire post-deletion timeline.}
    \label{fig:experiment-1-logistic-exponential-decay-rate}
\end{figure}

\begin{figure}
    \centering
    \includegraphics[width=1\linewidth]{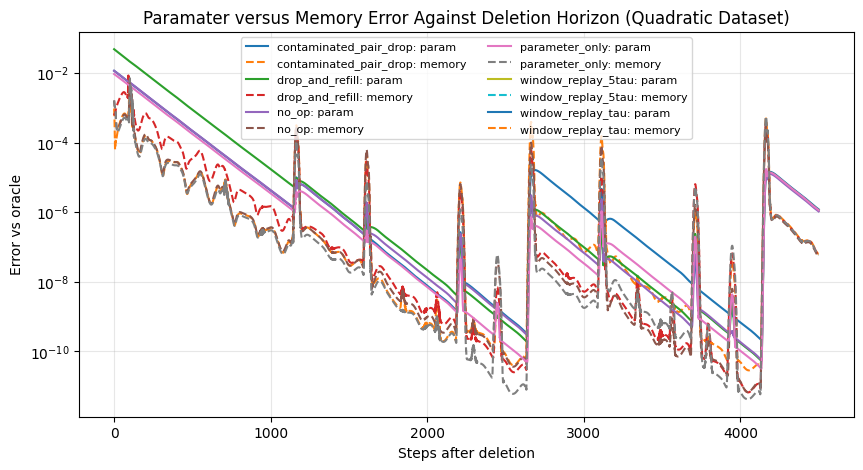}
    \caption{We chart the parameter and memory error for 4500 steps following the deletion at step $\tau=500$. The oracle and window replay interventions recover the counterfactual exactly, hence their error lines are out of scope of the visual. Initially, the drop and refill intervention shows the highest parameter and memory error, indicating that the deletion of useful curvature information hurts the optimizer's performance for the steps immeadiately following the deletion. Once this occurs, the contaminated pair drop and parameter only interventions gradually take the lead as the highest parameter errors. This is due to the misalignment effect: the parameter state is dynamically misaligned with the curvature state from persistent residual information. By the end of the simulation, the interventions with misalignment show the highest degree of error, surpassing even the error of no model correction at all. These unlearning interventions are not only inconsistent, but dynamically inconsistent and unstable within themselves.}
    \label{fig:experiment-1-quadratic-state-error-decomposition}
\end{figure}

Memory-only interventions have the complementary failure mode. Removing or resetting curvature information can reduce direct contamination, but it does not generally reconstruct the parameter trajectory that would have generated the counterfactual memory. In the aggregate results, full memory reset improves future state-error AUC relative to no-op deletion, but it still remains far from oracle recovery. Its final state error is approximately \(3.86 \times 10^4\), with nonzero future trajectory error and large update-direction error. Thus, memory reset can be beneficial as a damage-reduction operation, but it should not be interpreted as exact unlearning. It produces a state with clean or partially cleaned curvature memory attached to a parameter vector that still reflects the deleted trajectory.

The central lesson is that deleted information is not stored in one place. It is distributed across the coupled optimizer state. Successful unlearning therefore requires more than removing contaminated curvature pairs or perturbing parameters. It requires reconstructing a coherent state that could have been produced by the deletion-edited event stream.

\subsection{Finite curvature memory creates a pseudo-deletion horizon.}
The finite-memory decay experiment isolates how long deletion influence persists after the deletion event. Online L-BFGS stores only a finite number of curvature pairs, so directly contaminated curvature pairs eventually leave memory. This might suggest that deletion becomes automatic after $\tau$ additional updates. The results demonstrate otherwise.

Direct memory does clear after the contaminated curvature pairs rotate out of the L-BFGS window. However, state error does not necessarily vanish at that moment. The deleted data may already have changed the parameter trajectory, and that changed trajectory generates later gradients, later curvature pairs, and later inverse-Hessian actions.

This distinction explains the observed relationship between memory size and recovery. Larger curvature windows retain direct contamination for more steps, but they also slow the model's alignment. In the memory-depth experiment, parameter error and memory-operator error decay at different rates. Parameter error may decline smoothly, while memory-operator error is more sensitive to the number of stored curvature pairs. This indicates that the memory window $\tau$ controls the length of the optimizer's historical dependence.

The learning-unlearning cycle requires strategic amounts of forgetting. A larger memory window may improve optimization by retaining more curvature information, but it also increase the burden on the unlearning mechanism. A smaller memory window shortens the deletion horizon, but may weaken the optimizer's ability to fully utilize second-order structure. 

\subsection{Local deletions reduce some errors but do not reconstruct the counterfactual.}
The state-aware benchmark compares deletion-time interventions that modify parameters, memory, both, or neither. The aggregate results show a clear hierarchy. Oracle replay recovers the counterfactual perfectly. Among non-oracle methods, window replay over a larger window gives the lowest future state-error AUC, followed by full memory reset, retain fine-tuning, shorter window replay, contaminated-pair drop, and no-op deletion. The worst methods are drop-and-refill and parameter-only correction.

\begin{figure}
    \centering
    \includegraphics[width=1\linewidth]{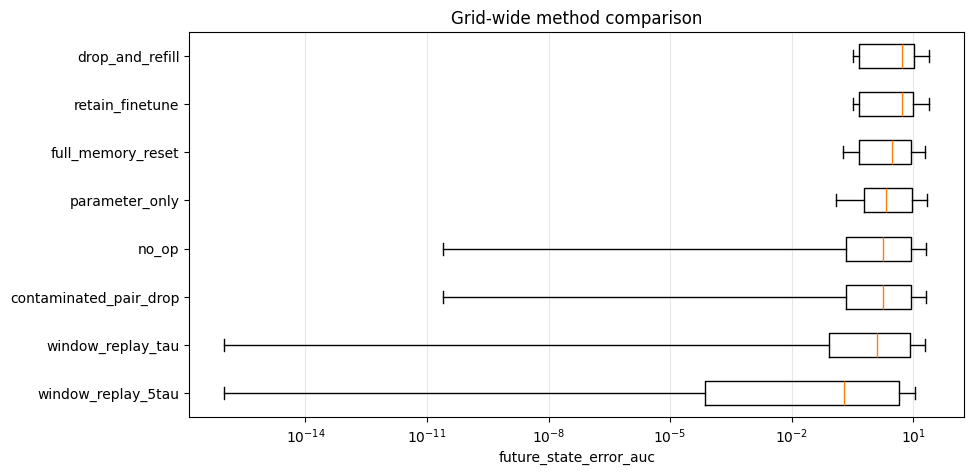}
    \caption{The efficiency of some unlearning methods varies substantially. Window replay interventions and contaminated pair drops work well when the deleted data is still in direct memory, preventing a heavy amount of residual information. When the deletions have passed from direct into indirect memory, the methods are more in alignment with the remaining methods of counterfactual approximation.}
    \label{fig:state-aware-grid-errors}
\end{figure}

The poor behavior of parameter-only correction is especially important. It is the intervention most closely aligned with standard approximate unlearning literature that evaluates unlearning with parameter distance. But in a stateful optimizer, this correction leaves the curvature memory unexamined. The resulting pair $(\tilde w_t, Z_t)$ is neither the original state nor the counterfactual, it's a hybrid state that is misaligned with itself. The empirical consequence is large future trajectory error and large future state-error AUC.

Drop-and-refill fails for a related reason. It discards both parameter and curvature information and then allows the optimizer to continue from the remaining stream. This removes contaminated information aggressively, but it also destroys useful information shared by the actual and counterfactual histories. Since the retain data substantially overlaps between the two histories, indiscriminate erasure can move the optimizer farther from the counterfactual than doing nothing. In the aggregate results, drop-and-refill has future state-error AUC approximately $4.80 \times 10^7$, substantially worse than no-op deletion.

Full memory reset and contaminated-pair drop perform better than no-op in the aggregate results, but their interpretation is different from replay. They reduce some memory-based error, especially when deleted data is still directly represented in the curvature window, but they do not reconstruct the counterfactual optimizer state. Their update-direction errors remain large, indicating that the induced future optimization behavior still differs from the counterfactual. These methods are therefore best understood as partial mitigation strategies. They can reduce residual contamination, but they do not solve the misaligned state.

\begin{figure}
    \centering
    \includegraphics[width=1\linewidth]{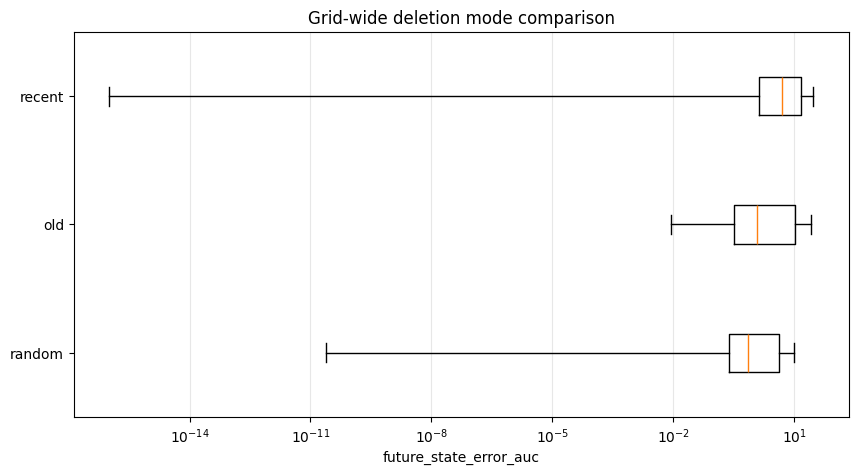}
    \caption{Older data is harder to forget. These data have almost certainly passed from direct into indirect memory, influencing the trajectory of later curvature points in a way that is difficult to discern.}
    \label{fig:experiment-1-deletion-mode-comparison}
\end{figure}

\subsection{Replay succeeds when strategically timed.}
Replay-based methods work differently than other interventions because they don't blindly adjust only a single model state. They regenerate the dual parameter-memory state from data that is consistent with the deletion request. This explains why window replay can approach exact recovery under specific conditions (ie. if the deleted data is a subset of the replay window). But the efficiency of replay methods reduces when deleted data passes from direct to indirect memory.

Window replay is a more realistic approximation of an oracle retrained model. Instead of replaying the entire stream, it maintains a recent window of observations and reconstructs the optimizer state from that window after deletion. In the aggregate benchmark, the larger window replay method has the best non-oracle future state-error AUC, approximately $1.49 \times 10^7$, and improves over both no-op deletion and the shorter replay window. It had the lowest future state AUC out of any non-oracle method in 40 of the 108 configuration experiments ran. 

\section{Discussion}
The experiments suggest that unlearning in stateful optimizers should be evaluated as a state reconstruction problem. Parameter-only correction can move \(w_t\) toward \(w_t^{-U}\), and memory-only correction can remove visibly contaminated curvature pairs from \(Z_t\), but neither operation guarantees that the resulting pair \((\tilde w_t,\tilde Z_t)\) lies on the counterfactual trajectory. This distinction matters because future o-LBFGS updates are generated by the interaction between the current gradient and the inverse-Hessian approximation induced by memory. A state that is only locally improved will induce an optimizer whose memory is inconsistent with its own parameters, inducing instability in parameter space.

The finite-memory structure of o-LBFGS creates a natural direct deletion horizon: curvature pairs generated from deleted samples eventually leave the memory window. This should not be confused with full unlearning. Direct clearance removes explicit references to deleted data, but the deleted samples may already have changed the parameter trajectory. Later curvature pairs are then generated from parameters that differ from the counterfactual parameters, preserving an indirect residue even after direct memory mass is zero.

The state-alignment view also clarifies what a certificate for stateful unlearning would need to control. Existing certified unlearning mechanisms typically calibrate noise to a sensitivity bound over parameters or empirical-risk minimizers. For o-LBFGS, the relevant sensitivity includes both the parameter and the memory-induced update operator. A state certificate should therefore bound deviations in a metric that captures future optimization behavior, not merely parameter distance.

\section{Limitations and Future Work}
This study is limited in four ways. First, the theoretical bounds rely on convexity, smoothness, bounded gradients, bounded inverse-Hessian approximations, and a contractive update assumption. These assumptions isolate the state-alignment mechanism but do not cover general nonconvex training. Second, the experiments use synthetic quadratic and ridge-logistic streams, which allow controlled counterfactual replay but do not yet demonstrate performance on large-scale deployed models. Third, the replay interventions assume access to sufficient event history. This may be infeasible in the case of potentially unbounded data streams and memory-constrained environments like IoT devices.

\newpage

\bibliographystyle{plain}
\bibliography{references}

\newpage

\appendix

\section{Convexity Assumptions}
\label{sec:convexity-assumptions}

\begin{assumption}[Strong convexity and smoothness]
\label{assump: strong-convexity}
Each loss $\ell_t(w)$ is $\mu$-strongly convex and $L$-smooth.
\end{assumption}

\begin{assumption}[Bounded stochastic gradients]
\label{assump: bounded-gradients}
There exists $G<\infty$ such that
\[
\|\nabla \ell_t(w)\|\le G
\]
for all relevant $t,w$.
\end{assumption}

\begin{assumption}[Bounded inverse-Hessian approximations]
\label{assump: bounded-hessian}
The o-LBFGS inverse curvature matrices satisfy
\[
m I \preceq H_t \preceq M I
\]
for constants $0<m\le M<\infty$.
\end{assumption}

\begin{assumption}[Contractive update map]
\label{assump: contractive-updates}
The optimizer update map $\Phi_t$ satisfies
\[
d_{\mathcal Z}(\Phi_t(Z),\Phi_t(Z'))\le \rho d_{\mathcal Z}(Z,Z')
\]
for some $\rho<1$.
\end{assumption}

\section{Methods}
We conduct two experiments. The finite-memory decay experiment measures how optimizer-state error decays after deleted curvature pairs leave the explicit o-LBFGS memory window. The state-aware unlearning benchmark compares deletion-time interventions that modify parameters, memory, both, or neither.

\begin{table}
\centering
\begin{tabular}{|c|c|}
\hline
Quantity & Default value \\
\hline
Dimension & \(d=25\) \\
Stream length, finite-memory decay & \(T=700\) \\
Deletion time, finite-memory decay & \(t_{\mathrm{del}}=300\) \\
Post-deletion horizon, finite-memory decay & \(H=250\) \\
Stream length, state-aware benchmark & \(T=5000\) \\
Deletion time, state-aware benchmark & \(t_{\mathrm{del}}=500\) \\
Post-deletion horizon, state-aware benchmark & \(H=4500\) \\
o-LBFGS memory size & \(\tau=10\) \\
Deletion size & \(|U|=5\) \\
Probe vectors & \(m_q=32\) \\
State memory weight & \(\lambda_Z=1\) \\
Ridge parameter, logistic regime & \(\lambda=0.05\) \\
\hline
\end{tabular}
\caption{Default experimental configuration. Grid experiments vary condition number, memory size, deletion mode, drift, and random seed as specified in the results.}
\label{tab:experiment-config}
\end{table}

\subsection{Synthetic Data-Generating Processes}
We evaluate both experiments on two synthetic online-learning regimes: drifting strongly convex quadratic losses and ridge-regularized logistic losses with Gaussian features. These regimes allow us to control curvature, condition number, drift, and deletion timing while preserving a known counterfactual replay procedure.

\paragraph{Drifting quadratic losses.}
For each event \(t\), the learner observes a quadratic loss
\[
\ell_t(w)
=
\frac{1}{2}(w-a_t)^\top H_t(w-a_t),
\]
where \(H_t \in \mathbb{R}^{d \times d}\) is symmetric positive definite. We construct \(H_t\) so that its eigenvalues lie in \([\mu,\kappa\mu]\), giving condition number \(\kappa\). The minimizer \(a_t\) follows a smooth drifting trajectory,
\[
a_t
=
a_0
+
\delta_a \sin(2\pi t/P_a)u_1
+
\delta_a \cos(2\pi t/P_a)u_2
+
\sigma_a \xi_t,
\]
where \(u_1,u_2\) are fixed unit directions, \(P_a\) is the minimizer-drift period, and \(\xi_t\) is isotropic Gaussian noise. When curvature drift is enabled, we interpolate between two independently generated positive definite matrices,
\[
H_t=(1-\alpha_t)H_0+\alpha_t H_1,
\qquad
\alpha_t=\frac{\delta_H}{2}\{1+\sin(2\pi t/P_H)\}.
\]
This setting isolates the effect of finite curvature memory under a controlled strongly convex loss landscape.

\paragraph{Ridge logistic losses with Gaussian features.}
For the logistic regime, each event consists of a feature-label pair \((x_t,y_t)\), where \(y_t\in\{-1,+1\}\). Features are sampled from a Gaussian distribution,
\[
x_t \sim \mathcal{N}(0,\Sigma_t),
\]
with \(\Sigma_t\) constructed to have condition number \(\kappa\). As in the quadratic case, optional curvature drift is induced by interpolating between two covariance matrices:
\[
\Sigma_t=(1-\alpha_t)\Sigma_0+\alpha_t\Sigma_1.
\]
Labels are generated from a drifting logistic model,
\[
\mathbb{P}(y_t=1\mid x_t)=\sigma(x_t^\top \beta_t),
\qquad
\beta_t=\beta_0+\delta_\beta \sin(2\pi t/P_\beta)v,
\]
where \(\sigma(z)=(1+e^{-z})^{-1}\). The online loss is
\[
\ell_t(w)
=
\log\{1+\exp(-y_t x_t^\top w)\}
+
\frac{\lambda}{2}\|w\|_2^2.
\]
The ridge term makes each loss \(\lambda\)-strongly convex, while the Gaussian feature covariance controls the conditioning of the observed gradients.

\subsection{State Alignment Metrics}
The o-LBFGS state at time \(t\) is
\[
\theta_t=(w_t,Z_t),
\]
where \(w_t\in\mathbb{R}^d\) is the parameter vector and \(Z_t\) is the finite memory of curvature pairs. For an intervened state \(\tilde\theta_t\) and the deletion-edited counterfactual state \(\theta_t^{-U}\), we measure three discrepancies.

\paragraph{Parameter error.}
The parameter error is the distance between the exact and counterfactual on some Euclidean metric.
\[
E_w(t)=\|\tilde w_t-w_t^{-U}\|_2.
\]

\paragraph{Memory-operator error.}
Because L-BFGS memory is used through inverse-Hessian-vector products, we compare memory states by their induced preconditioning actions. Let \(q_1,\ldots,q_{m_q}\) be fixed unit probe vectors shared across all methods within a run. We define
\[
E_Z(t)
=
\left(
\frac{1}{m_q}
\sum_{j=1}^{m_q}
\left\|
H_{\tilde Z_t}q_j
-
H_{Z_t^{-U}}q_j
\right\|_2^2
\right)^{1/2},
\]
where \(H_Zq\) denotes the inverse-Hessian action computed by the two-loop recursion used in o-LBFGS updates.

\paragraph{Combined state error.}
We report the weighted state discrepancy
\[
E_\Theta(t)=E_w(t)+\lambda_Z E_Z(t),
\]
where \(\lambda_Z\geq0\) controls the relative weight assigned to memory-operator misalignment.

\paragraph{Update-direction error.}
To measure whether two states induce the same future optimization behavior, we also compare their update directions. For future event \(e_t\),
\[
D_{\mathrm{upd}}(t)
=
1-
\frac{
d_t^\top d_t^{-U}
}{
\|d_t\|_2\|d_t^{-U}\|_2
},
\]
where
\[
d_t=-H_{\tilde Z_t}\nabla \ell_t(\tilde w_t),
\qquad
d_t^{-U}=-H_{Z_t^{-U}}\nabla \ell_t(w_t^{-U}).
\]
This quantity is zero when the two optimizer states induce identical update directions.

\subsection{Experiment 1: Finite-Memory Decay}
The finite-memory decay experiment tests whether removing deleted samples from the explicit o-LBFGS memory window is sufficient to align the optimizer with the deletion-edited counterfactual. The experiment compares two trajectories after a deletion event: the actual trajectory, which has already processed the deleted events, and the counterfactual trajectory, which is obtained by replaying the same prefix while skipping the deleted events.

For each run, we generate an event stream \((e_1,\ldots,e_T)\) and choose a deletion time \(t_{\mathrm{del}}\). We first train the actual optimizer on the full prefix \(e_{1:t_{\mathrm{del}}}\). We then select a deletion set \(U\subseteq \{1,\ldots,t_{\mathrm{del}}\}\) according to one of several deletion modes: recent deletions, old deletions, random deletions, or high-gradient deletions. The counterfactual optimizer is constructed by replaying the same prefix while skipping all events in \(U\).

After deletion, both trajectories are propagated on the same future stream \(e_{t_{\mathrm{del}}+1:t_{\mathrm{del}}+H}\). At each post-deletion step \(k\), we record parameter error \(E_w(k)\), memory-operator error \(E_Z(k)\), combined state error \(E_\Theta(k)\), update-direction error \(D_{\mathrm{upd}}(k)\), and direct memory mass
\[
M_{\mathrm{direct}}(k)
=
\sum_{(s_i,y_i)\in Z_{t_{\mathrm{del}}+k}}
\mathbf{1}\{ \mathrm{source}(s_i,y_i)\cap U\neq\emptyset\}.
\]
This separates explicit retention of deleted curvature pairs from indirect persistence through the parameter trajectory and newly generated curvature pairs.

We define the empirical direct-clearance time as
\[
\hat h_{\mathrm{direct}}
=
\min\{k\geq0: M_{\mathrm{direct}}(k)=0\}.
\]
To quantify post-clearance decay, we fit an exponential envelope to the state-error trajectory,
\[
E_\Theta(k)\approx C\rho^k,
\]
over specified post-deletion intervals and report \(\hat\rho\) as an empirical decay rate. This allows us to test whether larger o-LBFGS memory windows slow alignment even after directly contaminated curvature pairs have left memory.

\subsection{Experiment 2: State-Aware Unlearning Benchmark}
The state-aware unlearning benchmark evaluates whether deletion-time interventions that modify optimizer memory improve alignment with the deletion-edited counterfactual. Unlike Experiment 1, which measures natural decay after deletion, Experiment 2 applies explicit unlearning operators at \(t_{\mathrm{del}}\) and compares their future trajectories to oracle replay.

For each run, we use the following protocol.
\begin{enumerate}
    \item Generate an online event stream and train the actual o-LBFGS optimizer on the prefix \(e_{1:t_{\mathrm{del}}}\).
    \item Select a deletion set \(U\) from the prefix using the specified deletion mode.
    \item Construct the oracle counterfactual state \(\theta_{t_{\mathrm{del}}}^{-U}\) by replaying the prefix while skipping all events in \(U\).
    \item Apply each unlearning intervention to the actual state \(\theta_{t_{\mathrm{del}}}\), producing an intervened state \(\tilde\theta_{t_{\mathrm{del}}}^{(r)}\) for method \(r\).
    \item Propagate every intervened state and the oracle state on the same future event stream.
    \item Record initial, final, and cumulative trajectory discrepancies relative to the oracle.
\end{enumerate}

\begin{table}
    \centering
    \begin{tabular}{|c|>{\centering\arraybackslash}p{0.5\linewidth}|}\hline
         Intervention&  Description\\\hline
         Oracle Replay&  This is the gold standard. The model is retrained from scratch without the offending data, serving as a baseline.\\\hline
         No-Op Deletion&  The deletion is registered and excluded from future loss evaluations, but the parameter and optimizer states remain unchanged.\\\hline
 Parameter-Only& A Newton update is applied to the parameters, but the memory state $Z_t$ remains unchanged.\\\hline
 Retain Fine-Tuning& The model state remains the same as pre-deletion, but the update map is that of the counterfactual.\\\hline
 Full Memory Reset& The curvature points are reset, but the parameter state remains unchanged.\\\hline
 Contaminated Curvature Pair Drop& We track the curvature pairs added to the memory state. When the points undergo deletion, the corresponding curvature pairs are also removed from memory. They are not replaced after deletion, and are allowed to replace themselves.\\\hline
 Window Replay& Some window of observations $\tau$ is maintained. Once the deletions have been processed, the curvature points and parameter weights are retrained from the window of observations.\\\hline
 Drop and Refill& The parameter and curvature states of the model are both reset and allowed to retrain from the remaining stream events.\\ \hline
    \end{tabular}
    \caption{We perform 8 interventions to stress test the persistence of information following deletions. Oracle replay is the baseline, where the optimizer is retrained from scratch on the entire training set without the deleted data.}
    \label{tab:intervention-table}
\end{table}

\section{Experimental Data}
\begin{landscape}
\begin{table}[t]
\centering
\caption{
Global comparison across all grid configurations. 
The table reports median errors unless otherwise indicated. 
The oracle replay baseline defines the counterfactual target and is zero by construction. 
Among non-oracle methods, only the replay-based interventions reduce the parameter, memory, state, future-trajectory, and update-direction errors simultaneously. 
Parameter-only correction lowers the median initial parameter error, but its mean future state-error AUC is substantially larger than no-op, indicating heavy-tailed failures when the curvature memory remains inconsistent with the corrected parameter vector.
}
\label{tab:global-method-comparison}
\small
\begin{tabularx}{\textwidth}{lrrrrrrrr}
\toprule
Method 
& \makecell{Initial\\param.}
& \makecell{Initial\\memory}
& \makecell{Initial\\state}
& \makecell{Future\\state AUC}
& \makecell{Mean future\\state AUC}
& \makecell{Final\\state}
& \makecell{Update-dir.\\AUC}
& \makecell{Exact\\recovery} \\
\midrule
Oracle replay 
& 0.000 & 0.000 & 0.000 & 0.000 & 0.000 & 0.000 & 0.000 & 1.000 \\
No-op deletion 
& 185.509 & 0.122 & 187.849 & 444022.295 & 16904510.339 & 153.498 & 416.303 & 0.000 \\
Parameter-only 
& 100.245 & 0.122 & 122.879 & 436635.702 & 56162902.333 & 106.743 & 410.780 & 0.000 \\
Contaminated pair drop 
& 185.509 & 0.014 & 187.853 & 448283.206 & 15874059.253 & 45.583 & 419.470 & 0.000 \\
Full memory reset 
& 185.509 & 8.330 & 196.092 & 445062.597 & 15326664.006 & 174.452 & 431.021 & 0.000 \\
Retain fine-tuning 
& 135.252 & 0.043 & 138.433 & 424588.071 & 15415561.418 & 102.863 & 380.273 & 0.000 \\
Drop and refill 
& 97.938 & 0.043 & 106.727 & 438526.141 & 48038123.069 & 31.674 & 406.351 & 0.000 \\
Window replay $\tau$ 
& 0.007 & 0.002 & 0.010 & 1.069 & 15635043.548 & 0.000004 & 0.019 & 0.333 \\
Window replay $5\tau$ 
& 0.000 & 0.000 & 0.000 & 0.000 & 14929109.318 & 0.000 & 0.000 & 0.556 \\
\bottomrule
\end{tabularx}
\end{table}
\end{landscape}

\subsection{Phases of Post-Deletion Dynamics}

The state deviation between the observed and counterfactual decays at some geometric rate. Conceptually, we would expect the effect of the deleted information to decay as additional learning occurs, and the rate of decay would itself decrease as training progresses. We analyze this by dividing the post-deletion events into three categories:

\begin{enumerate}
    \item P1 is the post-deletion phase where the deleted data remains in the range of curvature pairs. 
    \item P2 is the post-deletion phase where the deleted data has passed from direct to indirect memory, but still remains within $2 \tau$ of the time of deletion. This is considered to be some period of reasonable indirect influence.
    \item P3 is the post-deletion phase that goes from $2 \tau$ to $T$ and encompasses the decay of deleted information within the indirect memory.
\end{enumerate}

\begin{figure}
    \centering
    \includegraphics[width=1\linewidth]{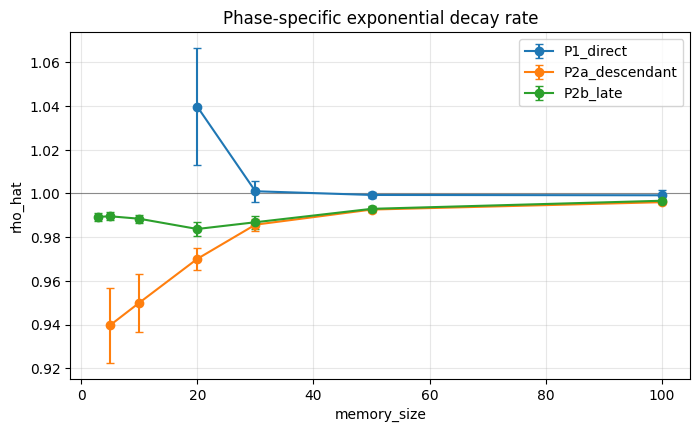}
    \caption{Phase-specific exponential decay rates for the quadratic stream. Although the quadratic loss surface is more controlled than the logistic setting, the direct post-deletion phase is not uniformly contractive. For intermediate memory depth $\hat \rho > 1,$ indicating the temporary amplification of state deviation. The later phases remain close to one, showing slow residual decay after direct clearance. So even under a convex quadratic objective, finite-memory o-LBFGS does not automatically erase deletion influence once contaminated curvature pairs rotate out.}
    \label{fig:experiment-1-quadratic-exponential-decay-rate}
\end{figure}

\begin{figure}
    \centering
    \includegraphics[width=1\linewidth]{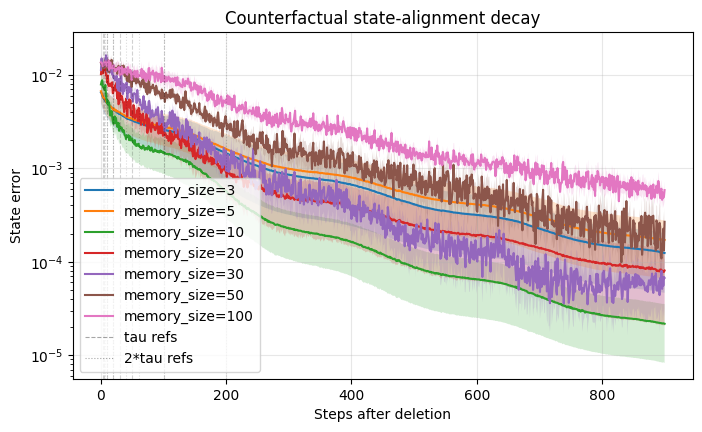}
    \caption{Counterfactual state-alignment error by memory depth after deletion. The $\tau$ and $2 \tau$ reference lines mark the nominal direct-memory and early indirect-memory horizons, but state deviation continues to decay long after these points. This shows that direct clearance of contaminated curvature pairs is not equivalent to full unlearning. Memory depth creates a tradeoff: very small windows clear historical information quickly but may lose useful curvature structure, while large windows preserve curvature information at the cost of slower counterfactual recovery. The persistence of error beyond $2 \tau$ is evidence of indirect memory in the optimizer trajectory.}
    \label{fig:experiment-1-quadratic-memory-size-variant}
\end{figure}

\begin{figure}
    \centering
    \includegraphics[width=1\linewidth]{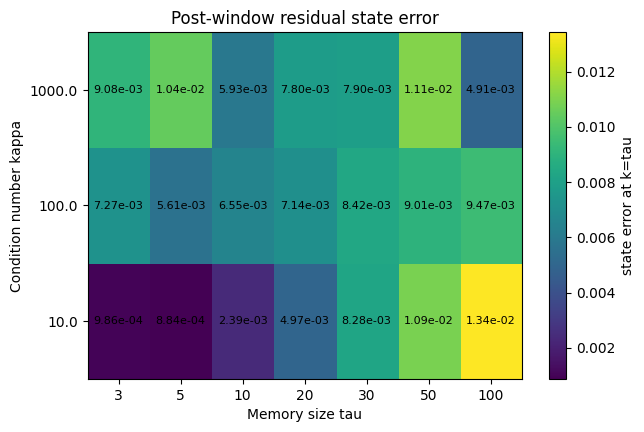}
    \caption{We also plot state deviations for the case of the quadratic loss surfaces under various degrees conditioning. Higher condition numbers indicate the model's strong sensitivity to a subset of its parameters.}
    \label{fig:experiment-1-quadratic-residual-error-heatmap}
\end{figure}

\begin{figure}
    \centering
    \includegraphics[width=1\linewidth]{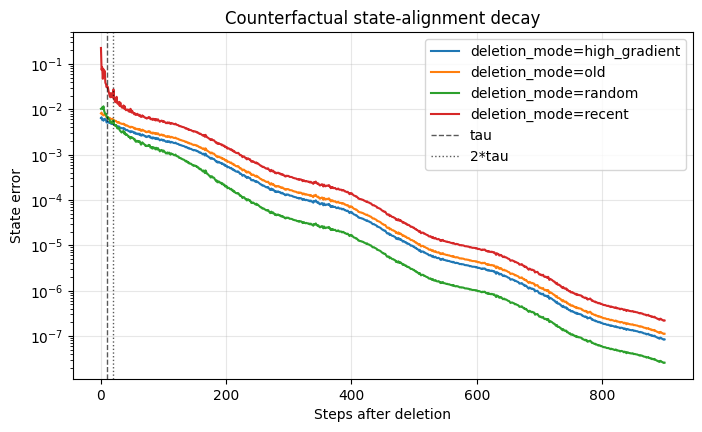}
    \caption{The state error is plotted for different deletion methods. The unlearning task is "easier" for more recent data, where the influence of the data has not persisted far into the optimizer's memory. Older deletions require a stronger degree of intervention, \textbf{indicating a functional incentive to accomodate deletion requests quickly} in the case of unlearning in large-scale production systems.}
    \label{fig:experiment-1-quadratic-sampling-intervention}
\end{figure}

\begin{figure}
    \centering
    \includegraphics[width=1\linewidth]{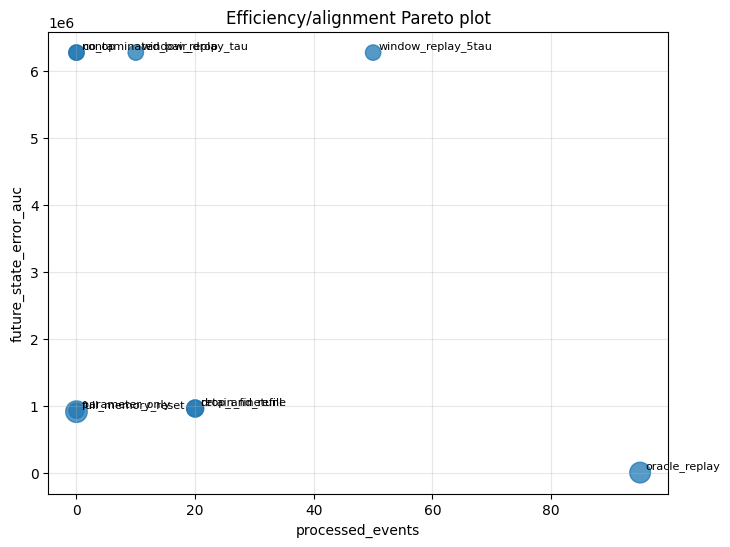}
    \caption{We plot the efficiency of different unlearning methods for a logistic loss surface. No method achieves true counterfactual recovery, but some approximations are more appropriate than others. The window replay interventions show the weakest performance, contrasting strongly with its performance in the quadratic loss case. In such cases, local state adjustments perform much better. This is likely due to the lack of a convex loss surface and a loss surface that shows a high degree of instability.}
    \label{fig:experiment-2-logistic-efficiency-plot}
\end{figure}

\begin{landscape}
\begin{table}[t]
\centering
\caption{
Evidence that reliable recovery requires counterfactual replay. 
Ratios are computed relative to no-op deletion within matched experimental configurations. 
Values below one improve on no-op. 
Although parameter-only correction beats no-op in a majority of configurations by median ratio, it never exactly recovers the future counterfactual trajectory and has a very large mean ratio, reflecting unstable failures. 
We see that window replay has the lowest future state error among non-oracle methods, supporting the claim that unlearning requires reconstruction of the coupled parameter-memory state rather than isolated component repair.
}
\label{tab:counterfactual-replay-evidence}
\small
\begin{tabularx}{\textwidth}{lrrrrr}
\toprule
Method 
& \makecell{Median AUC\\ratio vs. no-op}
& \makecell{Mean AUC\\ratio vs. no-op}
& \makecell{Share better\\than no-op}
& \makecell{Exact recovery\\rate}
& \makecell{Best non-oracle\\share} \\
\midrule
No-op deletion 
& 1.000 & 1.000 & 0.000 & 0.000 & 0.074 \\
Parameter-only 
& 0.971 & 16.985 & 0.574 & 0.000 & 0.102 \\
Contaminated pair drop 
& 1.000 & 1.064 & 0.333 & 0.000 & 0.000 \\
Full memory reset 
& 1.084 & 3.486 & 0.370 & 0.000 & 0.056 \\
Retain fine-tuning 
& 1.208 & 3.014 & 0.361 & 0.000 & 0.093 \\
Drop and refill 
& 1.310 & 22.908 & 0.361 & 0.000 & 0.028 \\
Window replay $\tau$ 
& 1.000 & 0.627 & 0.463 & 0.333 & 0.370 \\
Window replay $5\tau$ 
& 0.000 & 0.403 & 0.769 & 0.556 & 0.278 \\
\bottomrule
\end{tabularx}
\end{table}
\end{landscape}

\begin{landscape}
\begin{table}[t]
\centering
\caption{
Memory horizon and replay-depth evidence. 
The finite-memory hypothesis predicts that recovery should improve as the replay window covers a larger portion of the history that generated the current curvature state. 
The $5\tau$ intervention increasingly reconstructs the counterfactual state as memory size grows, reaching exact future-AUC recovery for all $\tau=20$ configurations. 
This supports the view that o-LBFGS admits a finite counterfactual recovery horizon, but that the relevant horizon is not merely the formal curvature window; it also includes enough preceding updates to regenerate the parameter trajectory that produced the current memory.
}
\label{tab:memory-horizon-replay-depth}
\small
\begin{tabularx}{\textwidth}{llrrrr}
\toprule
Memory size $\tau$ 
& Method 
& \makecell{Future\\state AUC}
& \makecell{AUC ratio\\vs. no-op}
& \makecell{Exact recovery\\rate}
& \makecell{Mean processed\\events} \\
\midrule
5 
& No-op deletion 
& 445839.058 & 1.000 & 0.000 & 0 \\
5 
& Parameter-only 
& 438049.480 & 0.866 & 0.000 & 0 \\
5 
& Contaminated pair drop 
& 450323.177 & 1.000 & 0.000 & 0 \\
5 
& Window replay $\tau$ 
& 1.053 & 1.000 & 0.333 & 5 \\
5 
& Window replay $5\tau$ 
& 1.052 & 0.764 & 0.333 & 25 \\
\midrule
10 
& No-op deletion 
& 449572.624 & 1.000 & 0.000 & 0 \\
10 
& Parameter-only 
& 436635.257 & 0.978 & 0.000 & 0 \\
10 
& Contaminated pair drop 
& 449572.623 & 1.000 & 0.000 & 0 \\
10 
& Window replay $\tau$ 
& 1.105 & 1.000 & 0.333 & 10 \\
10 
& Window replay $5\tau$ 
& 0.893 & 0.792 & 0.333 & 50 \\
\midrule
20 
& No-op deletion 
& 444021.687 & 1.000 & 0.000 & 0 \\
20 
& Parameter-only 
& 447984.148 & 0.986 & 0.000 & 0 \\
20 
& Contaminated pair drop 
& 448282.602 & 1.000 & 0.000 & 0 \\
20 
& Window replay $\tau$ 
& 1.148 & 0.910 & 0.333 & 20 \\
20 
& Window replay $5\tau$ 
& 0.000 & 0.000 & 1.000 & 100 \\
\bottomrule
\end{tabularx}
\end{table}
\end{landscape}

\end{document}